\documentclass{article}

\usepackage{microtype}
\usepackage{graphicx}
\usepackage{subfigure}
\usepackage{booktabs} 

\usepackage{hyperref}
\usepackage{url}
\usepackage{xcolor}
\usepackage{amsthm}
\usepackage{amsmath}
\usepackage{amssymb}
\usepackage{graphicx} 
\usepackage{booktabs} 
\usepackage{float}

\newtheorem{theorem}{Theorem}
\newtheorem{lemma}{Lemma}

\DeclareMathOperator{\divergence}{div}
\newcommand{\trace}[1]{\text{tr}\left[#1\right]}
\newcommand{\figref}[1]{\figurename~\ref{#1}}
\newcommand{\tableref}[1]{\tablename~\ref{#1}}
\usepackage[utf8]{inputenc}

\usepackage{hyperref}

\usepackage[accepted]{icml2020}

\icmltitlerunning{Equivariant Flows: Exact Likelihood Generative Learning for Symmetric Densities}

\begin{document}

\twocolumn[
\icmltitle{Equivariant Flows: Exact Likelihood Generative Learning for Symmetric Densities}



\icmlsetsymbol{equal}{*}

\begin{icmlauthorlist}
\icmlauthor{Jonas Köhler}{equal,fumacs}
\icmlauthor{Leon Klein}{equal,fumacs}
\icmlauthor{Frank Noé}{fumacs,fuphy,rice}
\end{icmlauthorlist}

\icmlaffiliation{fumacs}{Freie Universität Berlin, Department of Mathematics and Computer Science.}
\icmlaffiliation{fuphy}{Freie Universität Berlin, Department of Physics.}
\icmlaffiliation{rice}{Rice University, Department of Chemistry}

\icmlcorrespondingauthor{Jonas Köhler}{jonas.koehler@fu-berlin.de}
\icmlcorrespondingauthor{Frank Noé}{frank.noe@fu-berlin.de}

\icmlkeywords{Machine Learning, ICML}

\vskip 0.3in
]



\printAffiliationsAndNotice{\icmlEqualContribution} 

\begin{abstract}
Normalizing flows are exact-likelihood generative  neural networks which approximately transform  samples from a simple prior distribution to samples of the probability distribution of interest.  
Recent work showed that such generative models can be utilized in statistical mechanics to sample equilibrium states of many-body systems in physics and chemistry. 
To scale and generalize these results, it is essential that the natural symmetries in the probability density – in physics defined by the invariances of the target potential – are built into the flow.  
We provide a theoretical sufficient criterion showing that the distribution generated by \textit{equivariant} normalizing flows is invariant with respect to these symmetries by design. 
Furthermore, we propose building blocks for flows which preserve symmetries which are usually found in physical/chemical many-body particle systems.
Using benchmark systems motivated from molecular physics, we demonstrate that those symmetry preserving flows can provide better generalization capabilities and sampling efficiency.
\end{abstract}

\section{Introduction}
Generative learning using exact-likelihood methods based on invertible transformations has had remarkable success in accurately representing distributions of images \cite{kingma2018glow}, audio \cite{oord2017parallel} and 3D point cloud data \cite{LiuQiGuibas_FlowNet3D,noe2019boltzmann}.

Recently, \textit{Boltzmann Generators} (BG) \cite{noe2019boltzmann} have been introduced for sampling Boltzmann type distributions $\rho'(x) \propto \exp(-u(x))$ of high-dimensional many-body problems, such as valid conformations of proteins. 

This approach is widely applicable in the physical sciences, and has also been employed in the sampling of spin lattice states \cite{Nicoli_PRE_UnbiasedSampling,LiWang_PRL18_NeuralRenormalizationGroup} and nuclear physics models \cite{Albergo_PRD19_FlowLattice}. In contrast to typical generative learning problems, the target density $\rho'(x)$ is specified by definition of the many-body energy function $u(x)$ and the difficulty lies in learning to sample it efficiently. BGs do that by combining an exact-likelihood method that is trained to  approximate the Boltzmann density $\rho'(x)$, and a statistical mechanics algorithm to reweigh the generated density to the target density $\rho'(x)$. 

Physical systems of interest usually comprise symmetries, such as invariance with respect to global rotations or permutations of identical elements. As we show in experiments ignoring such symmetries in flow-based approaches to density estimation and enhanced sampling, e.g. using BGs, can lead to inferior results which can be a barrier for further progress in this domain. In our work we thus provide the following contributions:
\begin{itemize}
    \item We show how symmetry-preserving generative models, satisfying the exact-likelihood requirements of Boltzmann generators, can be obtained via \textit{equivariant flows}.
    \item We show that symmetry preservation can be critical for success by showing experiments on highly symmetric many-body particle systems. Concretely, equivariant flows are able to approximate the system's densities and generalize beyond biased data, whereas approaches based on non-equivariant normalizing flows cannot.
    \item We provide a numerically tractable and efficient implementation of the framework for many-body particle systems utilizing gradient flows derived from a simple mixture potential. 
\end{itemize}
While this work focuses mostly on applications in the physical sciences the results could provide a takeaway towards a greater ML audience: studying symmetries of target distributions and considering them in the architecture of a density estimation / sampling mechanism can lead to better generalization and can even be critical for successful learning.

\section{Related Work}
\label{sec:related-work}
\paragraph{Statistical Mechanics}
The workhorse for sampling Boltzmann-type distributions $p(x)\propto \exp(-u(x))$ with known energy function $u(x)$ are Molecular dynamics (MD) and Markov-Chain Monte-Carlo (MCMC) simulations. MD and MCMC take local steps in configurations $x$, are guaranteed to sample from the correct distribution for infinitely long trajectories, but are subject to the \textit{rare event sampling problem}, i.e. the get stuck in local energy minima of $u(x)$ for long time. Statistical mechanics has developed many tools to speed up rare events by adding a suitable bias energy to $u(x)$ and subsequently correcting the generated distribution by reweighing or Monte-Carlo estimators using the ratio of true over generated density, e.g. \citep{Torrie_JCompPhys23_187,Bennett_JCP76_BAR,LaioParrinello_PNAS99_12562,WuEtAL_PNAS16_TRAM}. These methods can all speed up MD or MCMC sampling significantly, but here we pursue sampling of the equilibrium density with flows.

\paragraph{Normalizing Flows}
Normalizing flows (NFs) are diffeomorphisms $f_{\theta} \colon \mathbb{R}^{n} \rightarrow \mathbb{R}^{n}$ which transform samples $z \sim \rho$ from a simple prior density $\rho$ into samples $x = f_{\theta}(z)$ \cite{tabak2010density, tabak2013family, rezende2015variational, papamakarios2019normalizing}.
Denoting the density of the transformed samples $\rho_{f_{\theta}}$, we obtain the probability density of any generated point via the \textit{change of variables} equation:
$$\rho_{f_{\theta}}(x) = \rho\left(f_{\theta}^{-1}(x)\right) \det \frac{\partial f_{\theta}^{-1}(x)}{\partial x}.$$
$\rho_{f_{\theta}}$ is also called the \textit{push-forward} of $\rho$ along $f_{\theta}$.

While flows can be used to build generative models by maximizing the likelihood on a data sample, having access to tractable density is especially useful in variational inference \cite{rezende2015variational, tomczak2016improving, louizos2017multiplicative, berg2018sylvester} or approximate sampling from distributions given by an energy function \cite{oord2017parallel}, which can be made exact using importance sampling \cite{muller2018neural, noe2019boltzmann}.

The majority of NFs can be categorized into two families: (1) Coupling layers \cite{dinh2014nice, dinh2016density, kingma2018glow, muller2018neural}, which are a subclass of autoregressive flows \cite{germain2015made, papamakarios2017masked, huang2018neural, de2019block, durkan2019neural}, and (2)
residual flows \cite{chen2018neural, zhang2018monge, grathwohl2018ffjord, behrmann2018invertible, chen2019residual}. 

Symmetries in flow models have been discussed in the
context of permutations in graphs \cite{liu2019graph}.
A preliminary account of equivariant normalizing flows has been given
in two recent workshop submissions \cite{rezende2019equivariant, kohler2019equivariant}.


\paragraph{Boltzmann-Generating Flows}\label{boltzmann-generating-flows}

While flows and other generative models are typically used for estimating the an unknown density $\rho'$ from samples and then generating new samples from it, BGs know the desired target density
$\rho'(x) \propto \exp(-u(x))$ up to a prefactor and aim at learning to efficiently sample it \cite{noe2019boltzmann}.

A BG combines two elements to achieve this goal:
\begin{enumerate}
    \item An exact-likelihood generative model that generates samples $x_k$ from a density $\rho_{f_{\theta}}$ that approximates the given Boltzmann-type target density $\rho'$.
    \item An algorithm to reweigh the generated density to the target density $\rho'$. For example, using importance sampling the asymptotically unbiased estimator
    of the expectation value of observable $O(x)$ is:
    $$
    \mathbb{E}_{x \sim \rho'}[O] \approx \frac{\sum_k w(x_k) O(x_k)}{\sum_k w(x_k)}, \quad x_{k} \sim \rho_{f_{\theta}},
    $$
    where the importance weights $$w(x_k) = \exp(-u(x_k))/\rho_{f_{\theta}}(x_k)$$ can be computed from the trained flow.
\end{enumerate}
The exact likelihood model is needed in order to be able to conduct the reweighing step. When a flow is used in order to generate asymptotically unbiased samples of the target density, we speak of a Boltzmann-generating flow.

Boltzmann-generating flows are trained to match $\rho_{f_{\theta}} \approx \rho'$ using loss functions that also appear in standard generative learning problems, but due to the explicit availability of $\exp(-u(x))$ their functional form and interpretation changes:
\begin{enumerate}
\item \textit{KL-training} We minimize the reverse Kullback-Leibler divergence $KL(\rho_{f_{\theta}}\|\rho')$:
$$
\mathcal{L}_{\text{KL}} = \mathbb{E}_{z\sim \rho} \left[
u(f_{\theta}(z))
- \log\left|\det\frac{\partial f_{\theta}(z)}{\partial z}\right|
\right].
$$
This approach is also known as energy-based training where the energy corresponding to the generated density is matched with $u(x)$.

\item \textit{ML-training}: If data $\left\{x_{n}\right\}_{n=1 \ldots N}$ from
a data distribution $\rho'_{\text{data}}$ is given that at least represents one or a few high-probability modes of $\rho'$, we can maximize the likelihood under the model, as is typically done when performing density estimation:
\begin{align*}
\mathcal{L}_{\text{ML}} &= \mathbb{E}_{x\sim \rho'_{\text{data}}} \bigg[
-\log \rho\left( f_{\theta}^{-1}(x) \right)\\
& \qquad \qquad \quad
- \log\left|\det\frac{\partial f_{\theta}^{-1}(x)}{\partial x}\right|
\bigg].
\end{align*}
\end{enumerate}
The final training loss is then obtained using a convex sum over both losses, where the mixing parameter $\lambda$ may be changed from $0$ to $1$ during the course of training:
$$
    \mathcal{L} = (1-\lambda)\mathcal{L}_{\text{ML}} + \lambda \mathcal{L}_{\text{KL}}.
$$

\section{Invariant Densities via Equivariant Flows}

In this work we consider densities $\rho, \rho'$ over euclidean vector spaces $\mathbb{R}^{n}$ which are invariant w.r.t. to symmetry transformations e.g. given by rotations and permutations of the space. In other words, we want to construct flows such that both, the prior and the target density share the same symmetries.

More precisely, let $G$ be a group which acts on $\mathbb{R}^{n}$ via a representation $R \colon G \rightarrow GL(n), g \rightarrow R_{g}$ and assume that $\rho$ is invariant w.r.t. $G$, i.e. $\forall g \in G, x \in \mathbb{R}^{n} \colon \rho(R_{g} x) = \rho(x)$. We first remark that for any $g \in G$ the matrix $R_g$ satisfies $\det(R_g) \in \{-1, 1\}$\footnote{All proofs and derivations can be found in the Suppl. Material.}.
This allows us to formulate our result:
\begin{theorem}\label{thm:sufficient-criterion-classic-flow}
    Let $\rho$ is a density on $\mathbb{R}^{n}$ which is $G$-invariant and $G > H$. If $f$ is a $H$-equivariant diffeomorphism, i.e. $\forall h \in H, x \in \mathbb{R}^{n} \colon f(R_{h} x) = R_{h} f(x)$, then $\rho_{f}$ is $H$-invariant.
\end{theorem}
As a direct consequence if $H < O(n)$, any push-forward of an isotropic normal distribution along a $H$-equivariant diffeomorphism will result in a $H$-invariant proposal density.

\section{Constructing Equivariant Flows}

In general it is not clear how to define equivariant diffeomorphisms which provide tractable inverses and Jacobians. We will provide a possible implementations based on the recently introduced framework of \textit{continuous normalizing flows} (CNFs) \cite{chen2017continuous}.

\paragraph{Equviariant Dynamical Systems}
CNFs define a dynamical system via a time-dependent vector field $v\colon \mathbb{R}^{n} \times [0, \infty) \rightarrow \mathbb{R}^{n}$. If $v$ is globally Lipschitz, we can map each $z \in \mathbb{R}^{n}$ onto the unique characteristic function $x_{v, z} \colon [0, \infty) \rightarrow \mathbb{R}^{n}$, which solves the Cauchy-problem 
\begin{align*}
    \tfrac{d}{dt} x(t) &= v(x_{v, z}(t), t), \qquad x_{v, z}(0) = z.
\end{align*}
This allows us to define a bijection ${F_{v, T} \colon \mathbb{R}^{n} \rightarrow \mathbb{R}^{n}}$ for each $T \in [0, \infty)$ by setting 
\begin{align*}
F_{v, T}(z) = x_{v, z}(0) + \int_{0}^{T} dt ~ v(x_{v, z}(t), t).
\end{align*}
Given a density $\rho$ on $\mathbb{R}^{n}$, each $T$ defines a push-forward $\rho_{F_{v, T}}$ along $F_{v, T}$, which satisfies 
\begin{align*}
    \tfrac{d}{dt} \log \rho_{F_{v, t}}(x_{v, z}(t)) = - \text{div}\left(v(x_{v, z}(t), t)\right).
\end{align*}
By following the characteristic this allows to compute the total density change as 
\begin{align*}
\log \frac{\rho_{F_{v, T}}(x_{v, z}(T))}{\rho(x_{v, z}(0))} = - \int_{0}^{T} dt ~ \divergence \left(v(x_{v, z}(t), t)\right).
\end{align*}

Equivariant flows can thus be constructed very naturally:
\begin{theorem}\label{thm:sufficient-criterion-continuous-flow}
    Let $v$ be a $H$-equivariant vectorfield on $\mathbb{R}^{n}$ (not necessarily bijective). Then for each $T \in [0, \infty)$ the bijection $F_{v, T}$ is $H$-equivariant. 
\end{theorem}
Consequently, if $\rho$ is a $G$-invariant density on $\mathbb{R}^{n}$ and ${G > H}$, then each push-forward $\rho_{F_{v, T}}$ is $H$-invariant. 

\paragraph{Equivariant Gradient Fields}
 There has been a significant amount of work in recent years proposing $G$-equivariant functions for different groups acting on $\mathbb{R}^{n}$. A generic implementation however is given by a gradient flow: if $\Phi \colon \mathbb{R}^{n} \rightarrow \mathbb{R}$ is a $G$-invariant function, the vector $\nabla_{x} \Phi$ will transform $G$-equivariantly.

Gradient flows (not necessarily $G$-equivariant) can map any $\rho$ onto any $\rho'$ over $\mathbb{R}^{n}$ as long as both densities do not vanish \cite{benamou2000computational, mccann2001polar} and have been discussed in the context of density estimation \cite{zhang2018monge, papamakarios2019normalizing}. 

\paragraph{Numerical Implementations}
While providing an elegant solution, implementing equivariant flows using continuous gradient flows is numerically challenging due to three aspects. 

First, even if $F_{v, T}$ is invertible assuming exact integration, there are no such guarantees for any discrete-time approximation of the integral, e.g. using Euler or Runge-Kutta integration. Thus, \citeauthor{chen2017continuous} propose adaptive-step solvers, such as Dopri5 \cite{dormand1980family}, which can require hundreds of vector field evaluations to reach satisfying numerical accuracy.

Second, in order to train $v$ via the adjoint method as suggested by \citeauthor{chen2017continuous}, gradients of the loss w.r.t. parameters are obtained via backward integration. However, in general, there are no guarantees that this procedure is stable, which therefore can result in very noisy gradients, leading to long training times and inferior final results \cite{gholami2019anode}. In contrast to this \textit{optimize-then-discretize} (OTD) approach, \citeauthor{gholami2019anode} suggest to unroll the ODE into a fixed-grid sequence and backpropagate the error using classic automatic differentation (AD). Such a \textit{discretize-then-optimize} (DTO) approach will guarantee that gradients are computed correctly, but might suffer from inaccuracy due to the discretization errors as mentioned before. Throughout our experiments, we rely on the latter approach during training and show that for our presented architecture OTD and DTO will yield similar results, while the latter offers a significant speedup per iteration, more robust training and faster convergence.

Finally, computing the divergence of $v$ using off-the-shelf AD frameworks requires $O(n)$ backpropagation passes, which would result in an infeasible overhead for high-dimensional systems \cite{grathwohl2018ffjord}. Thus, \citeauthor{grathwohl2018ffjord} suggest an approximation via the Hutchinson-estimator \cite{Hutchinson1989ASE}. 
This is an unbiased rank-1 estimator of the divergence where variance scales with $O(n)$. 
As we show in our experiments, even for small particle systems, relying on such an estimator will render importance weighing and thus the benefits of Boltzmann generating flows useless, e.g. when used in downstream sampling applications. 
Another approach relies on designing special dynamics functions, in which input dimensions are decoupled and then combine the \texttt{detach}-operator with one backpropagation pass to compute the divergence exactly \cite{chen2019neural}.
For general symmetries as studied in this paper such a decoupling is not possible, without either destroying equivariance of the dynamics function, or enforcing it to be trivial.
Our proposed vector field based on a simple mixture of Gaussian radial basis functions (RBF)  allows computing the divergence numerically exact as one vectorized operation and without relying on AD backward passes.

\paragraph{Relation to Hamiltonian Flows}
If our space decomposes as $\mathbb{R}^{n} = \mathbb{R}^{m} \bigoplus \mathbb{R}^{m}$ where each element is written as ${x = (q, p)}$ and where we call $q$ the generalized position and $p$ the generalized momentum, we can define a time-dependent Hamiltonian $\mathcal{H} \colon \mathbb{R}^{m} \times \mathbb{R}^{m} \times [0, \infty)  \rightarrow \mathbb{R}$, which defines the Hamiltonian system
$$
v(q, p, t) = \left( \frac{\partial \mathcal{H}(q, p, t)}{\partial p}, -\frac{\partial \mathcal{H}(q, p, t)}{\partial q}  \right).
$$
If $\mathcal{H}$ factorizes as $\mathcal{H}(q, p, t) = V(q, t) + \tfrac{1}{2} \|p\|^2$ a numerically stable and finite-time invertible solution of the system is given by \textit{Leapfrog-integration}. Furthermore, due to the symplecticity of $v$, each $F_{v, T}$ will be volume preserving. Unrolling the Leapfrog-integration in finite time, will result in a stack of \textit{NICE}-layers \cite{dinh2014nice} with equivariant translation updates. 

We can always create an artificial Hamiltonian version of any density estimation problem, by augmenting a density $\rho(q)$ on $\mathbb{R}^{n}$ to $\rho(q, p) = \rho(q) \cdot \rho(p | q)$ on $\mathbb{R}^{n} \times \mathbb{R}^{n}$. Due to the interaction between $q$ and $p$ within the flow, we cannot expect that both $\rho(p | q) = \rho(p)$ and $\rho'(p | q) = \rho'(q)$ within a finite number of steps. Thus, if an isotropic normal distribution is used for $\rho(p, q)$, having only access to $\rho'(q)$ will require a variational approximation of $\rho'(p | q)$ \cite{toth2019hamiltonian}. 
\newpage
If $\mathcal{H}$ is $G$-invariant, i.e. $\mathcal{H}(R_{g} q, R_{g} p, t) = \mathcal{H}(q, p, t)$ for all $g \in G, (q,p) \in \mathbb{R}^{m}\times\mathbb{R}^{m}, t \in [0, \infty)$, we see that $v$ will be $G$-equivariant. This results in the recently proposed framework of \textit{Hamiltonian Equivariant Flows} (HEF) \cite{rezende2019equivariant}, which we thus see as a special case of our framework for densities with linearly represented symmetries defined over $\mathbb{R}^{n}$. On the other hand, HEFs can handle more general spaces or symmetries with nonlinear representations -- in contrast to the presented framework -- hence the two approaches are complementary.

For completeness, we note that Hamiltonian flows do not suffer from those numerical complications in the former paragraph, due to symplectic integration and volume preservation. However, in order to compute unbiased estimates of target densities which is essential for physics applications, a variational approximation of $\rho'(p|q)$ cannot be applied.


\section{Sampling of Coupled Particle Systems}

We evaluate the importance of incorporating symmetry into flows when aiming to sample from symmetric densities, by applying the theoretic framework to the problem of sampling coupled many-body systems of interchangeable particles.
Such systems have states $x \in \mathbb{R}^{n}, ~ n=N \cdot D$ consisting of $N$ particles $x_{i}$ with $D \in [2,3]$ degrees of freedom, which are coupled via a potential energy $u(x)$. In thermodynamic equilibrium such a system follows a Boltzmann-type distribution $\rho'(x) \propto \exp(-u(x)).$ Assuming interchangeable particles in vacuum without external field, we obtain three symmetries (S1-3): $u$ (and thus $\rho'$) does not change if we permute particles (S1), rotate the system around the center of mass (CoM) (S2), or translate the CoM by an arbitrary vector (S3).

Due to the simultaneous occurrence of (S1) and (S2) no autoregressive decomposition / coupling layer can be designed to be equivariant. Either a variable split has to be performed among particles or among spatial coordinates, which will break permutation and rotation symmetry respectively. Thus, residual flows are the only class of flows which can be applied here. In this work we will rely on CNFs, design an equivariant vector field by taking the gradient field of an invariant potential function, and then combine theorems 1 and 2 to conclude the symmetry of the proposal density.

\paragraph{Invariant Prior Density}
We first start by designing an invariant prior. By only considering systems with zero CoM symmetry (S3) is easily satisfied. The set of CoM-free systems forms a $(N-1)\cdot D$-dimensional linear subspace $U < \mathbb{R}^{n}$. Equipping $\mathbb{R}^{n}$ with an isotropic normal density $\rho$, implicitly equips $U$ with a normal distribution $\tilde \rho$. We can sample it, by sampling $z \sim \rho$ and projecting on $U$, and evaluate its likelihood for $z \in U$, by computing $\rho(z)$.

\paragraph{Equivariant Vector Field}

We design our vector field as the gradient field $v(x(t)) = \nabla_{x(t)} \Phi(x(t))$ of a potential $\Phi \colon \mathbb{R}^{n} \rightarrow \mathbb{R}$. If $\Phi$ is invariant under symmetry transformations (S1-3) it directly implies equivariance of $v$. 

Our invariant potential $\Phi$ is given as a sum of pairwise couplings over particle distances:
\begin{align*}
    \Phi(x(t)) = \sum_{ij} \tilde \Phi(d_{ij}(t),t)
\end{align*}
with $r_{ij}(t) = x_{i}(t) - x_{j}(t), ~ d_{ij}(t) = \|r_{ij}(t)\|$.
This yields per-particle updates
\begin{align*}
    v_{i}(x(t)) &= \sum_{j} v_{ij}(x(t)). 
\end{align*}
For a well-chosen coupling potential $\tilde \Phi(d_{ij}, t)$ we can express
\begin{align}
    v_{ij}(x(t)) &= \underbrace{R(t)^T W  K(d_{ij}(t))}_{\phi(d_{ij}) } \cdot r_{ij}(t),
    \label{eq:gradient-field}
\end{align}
where $K \colon \mathbb{R} \rightarrow \mathbb{R}^{M}$ and $R \colon \mathbb{R} \rightarrow\mathbb{R}^{L}$ are vector-valued functions, each component is given by a Gaussian RBF
and $W \in \mathbb{R}^{T \times M}$ is a trainable weight matrix (see \figref{fig:architecture}). 



Using this architecture, the divergence becomes: 
\begin{align*}
\label{eq:analytic-form-log-density}
    \divergence \frac{\partial x(t)}{\partial t} &= 
    \sum_{ij} \frac{\partial \phi(d_{ij}(t), t)}{\partial d_{ij}(t)} d_{ij}(t) + D \cdot \phi(d_{ij}(t)).
\end{align*}
Thus, the gradient and the divergence can be computed \textit{exactly} and \textit{as one vectorized operation} (see Suppl. Material for details).

During training we optimize $W$ and RBF means and bandwidths simultaneously. By keeping weights small and bandwidths large we can control the complexity of the dynamics. As we show in our experiments even a small amount of weight-decay is sufficient to properly optimize the flow with a fixed-grid solver introducing a negligible amount of error during the integration.

\paragraph{Other Invariant Potential Functions}
While $\Phi$ could be modeled by any kind of invariant graph neural networks, such as \textit{SchNet} \cite{schutt2017schnet}, this would require us to 1) use AD in order to compute $\nabla_{x(t)} \Phi(x(t))$ and 2) compute $\Delta_{x(t)} \Phi(x(t))$ at every function evaluation while integrating $v$. This implies the numerical challenges as mentioned before. As we show in the Suppl. Mat. our simple couplings are considerably faster, have a fraction of parameters while consistently outperforming neural network approaches to modeling $\Phi$ for the studied target systems.

\begin{figure}[h]
    \vskip 0.2in
    \begin{center}
    \centerline{\includegraphics[width=\columnwidth]{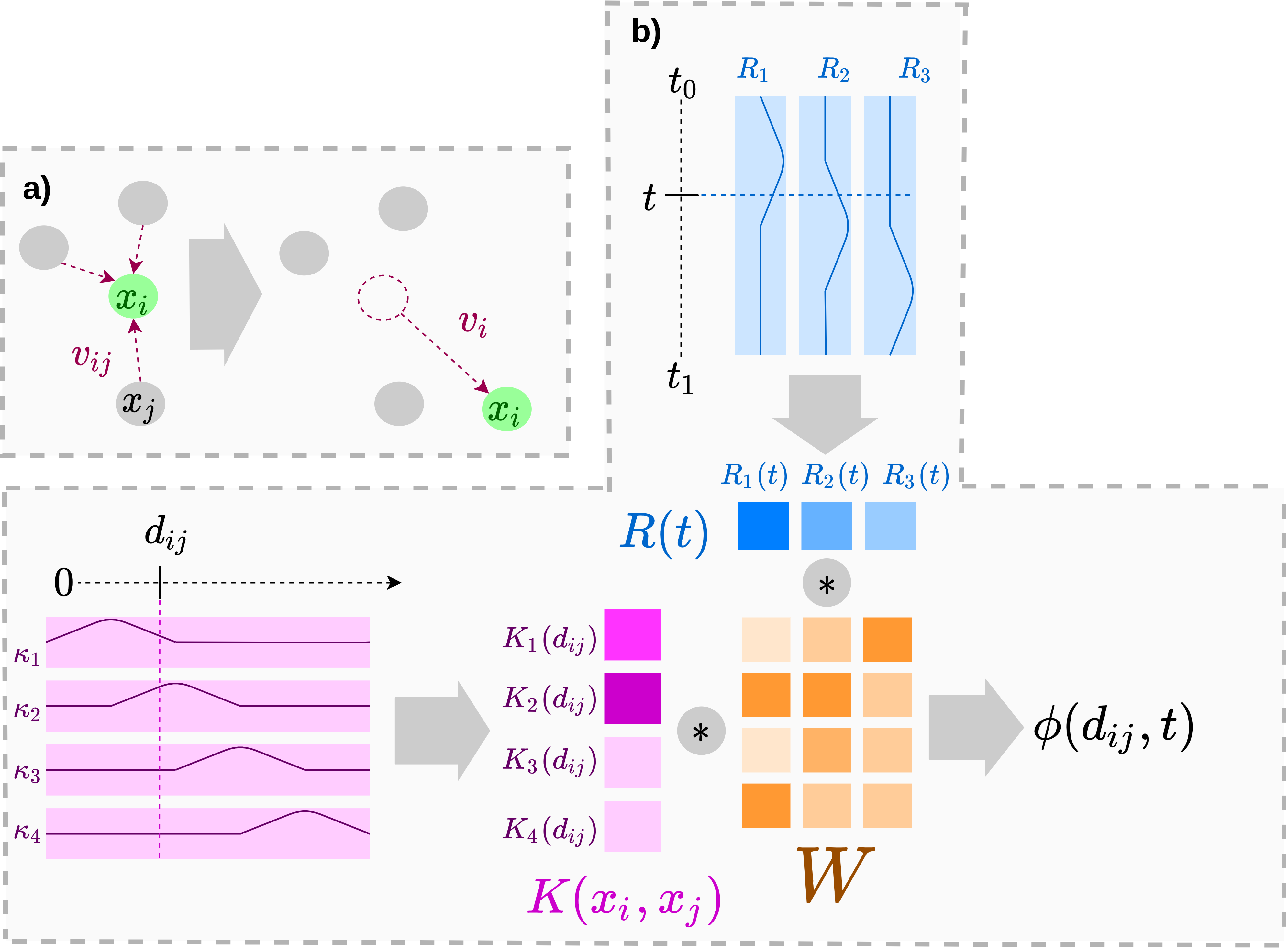}}
    \caption{\textbf{a)} Each particle $x_{i}$ is updated by a weighted sum of radial forces depending on distances and the integration time. \textbf{b)} Time and distances are expanded in a RBF basis and mixed together with a shared weight matrix.}
    \label{fig:architecture}
    \end{center}
     \vskip -0.2in
\end{figure}

\newpage

\section{Benchmark Systems}

We study two systems where all symmetries (S1), (S2), (S3) are present (\figref{fig:target_systems}):

\paragraph{DW-2 / DW-4}
The first system is given by $N \in [2, 4]$ particles with a pairwise \textit{double-well} potential acting
on particle distances
$$
u^{\texttt{DW}}(x) = \frac{1}{2 \tau}\sum\limits_{i,j}
a\, (d_{ij} - d_0) + b\, (d_{ij} - d_0)^2  + c\, (d_{ij} - d_0)^4
$$
for $D=2$, which produces two distinct low energy modes separated by an energy barrier. By coupling multiple particles with such double-well interactions we can create a frustrated system with multiple metastable states. Here $a, b, c$ and $d_0$ are chosen design parameters of the system and $\tau$ the dimensionless temperature.

\paragraph{LJ-13}

The second system is given by the \textit{Lennard-Jones} (LJ) potential with $N=13, ~ D=3$. LJ is a model for solid-state models and rare gas clusters. LJ clusters have complex energy landscapes whose energy minima are difficult to find and sample between. These systems have been extensively studied \cite{wales97LJ} and are good candidates for benchmarking structure generation methods.  
In order to prevent particles to dissociate from the cluster at the finite sampling temperature, we add a small harmonic potential to the CoM.
The LJ potential with parameters $\epsilon$ and $r_{m}$ at dimensionless temperature $\tau$ is defined by 
$$
u^{\texttt{LJ}}(x) = \frac{\epsilon}{2 \tau} \left[ \sum_{i,j} \left( \left(\frac{r_{m}}{d_{ij}}\right)^{12} - 2 \left(\frac{r_{m}}{d_{ij}}\right)^{6} \right)\right].
$$

\begin{figure}[h]
    \vskip 0.2in
    \begin{center}
    \centerline{\includegraphics[width=\columnwidth]{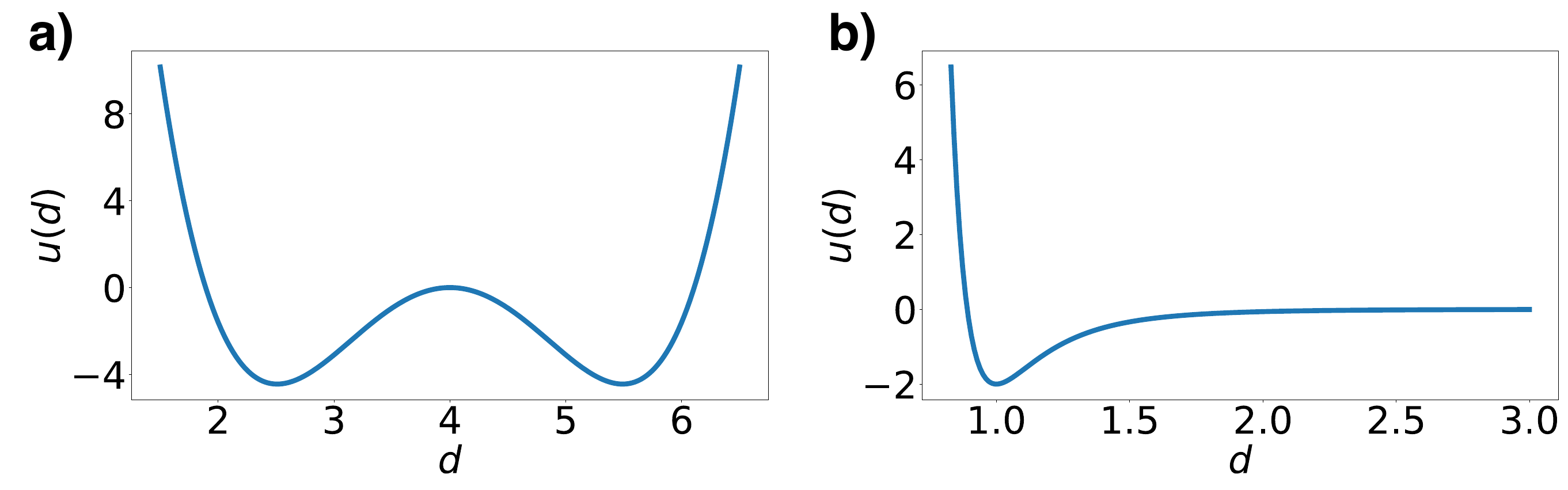}}
    \caption{The two model systems: shown are the energy contributions per distance \textbf{a)} for the \textit{double-well} and  \textbf{b)} the \textit{Lennard-Jones} potential. }
    \label{fig:target_systems}
    \end{center}
     \vskip -0.2in
\end{figure}

\section{Experiments}


\subsection{Computation of Divergence}
In a first experiment we show that \textit{fast} and \textit{exact} divergence computation can be critical especially when the number of particles grows. 
We compare different ways to estimate the change of log-density: (1) using brute-force computation relying on AD (2) using the Hutchinson estimator described by \citeauthor{grathwohl2018ffjord}, and (3) computing the trace exactly in close form.

Brute-force computation quickly yields a significant overhead per function evaluation during the integration, which makes it impractical for online computations (\figref{fig:trace} c), such as using the flow within a sampling procedure or just for training. If we use Hutchinson estimation, the error grows quickly with the number of particles (\figref{fig:trace} a) and renders reweighing, even for the very simple DW-2 system,  impossible (\figref{fig:trace} b). By having access to an exact closed-form trace, we obtain the best of both worlds: fast computation and the possibility for exact reweighing (\figref{fig:trace} b+c).

\begin{figure}[h]
    \vskip 0.2in
    \begin{center}
    \centerline{\includegraphics[width=\columnwidth]{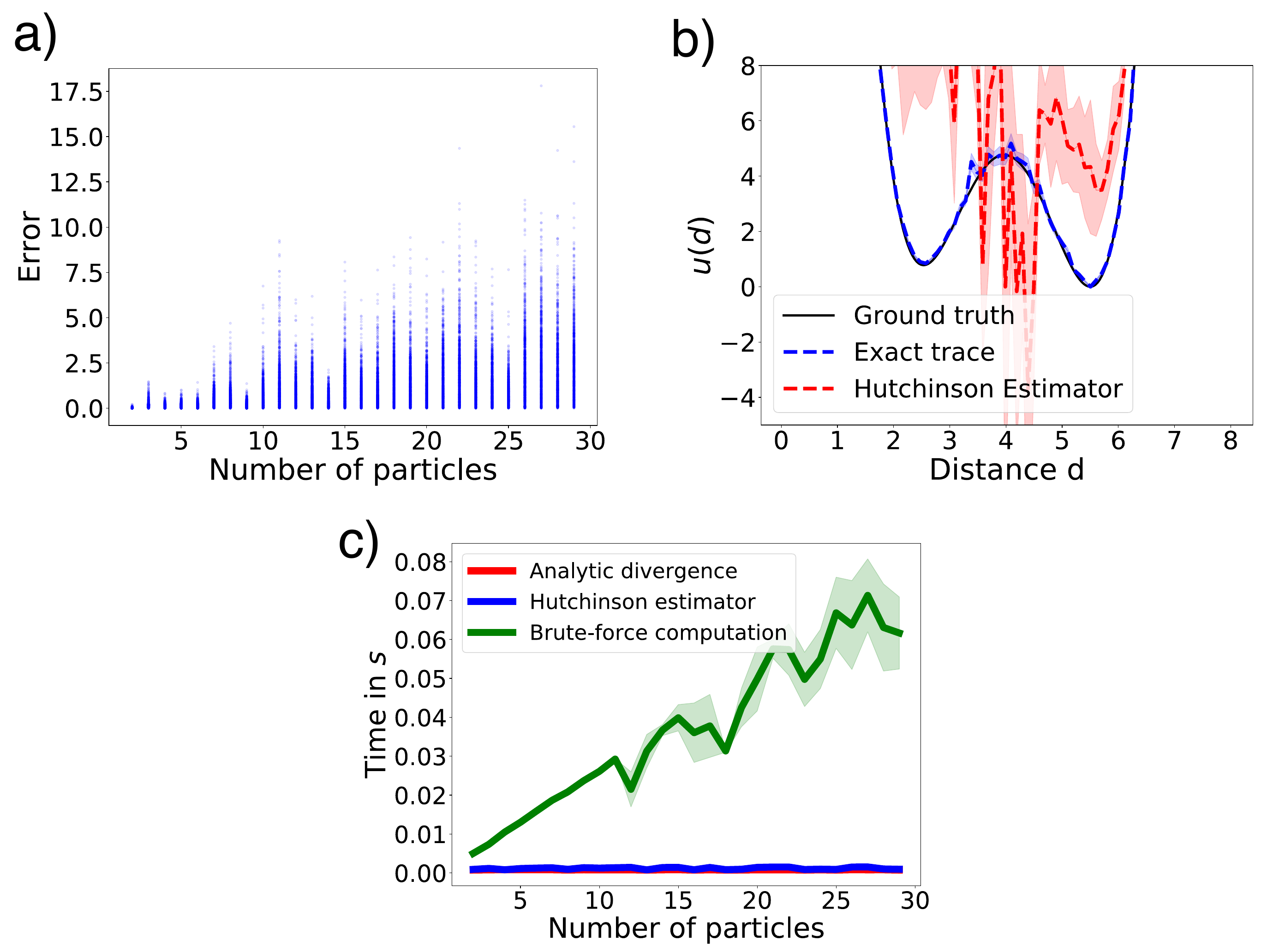}}
    \caption{\textbf{a)} Error of divergence estimates introduced by the Hutchinson estimator with growing number of particles.
    \textbf{b)} Free-energy profile of a DW-2 potential and importance-weighed estimates. 
    \textbf{c)} 
    Wall-clock time of evaluating $v(x(t), t)$ with growing number of particles.
    }
    \label{fig:trace}
    \end{center}
     \vskip -0.2in
\end{figure}

\subsection{DTO vs. OTD Optimization}

In this experiment we show that by simply regularizing $W$, e.g. using weight decay, OTD and DTO based optimization of the flow barely shows any difference (\figref{fig:otd_dto} a), while the former quickly results in a significant overhead due to the increasing number of function evaluations required to match the preset numerical accuracy (\figref{fig:otd_dto} b). We compare the OTD implementation presented in \cite{chen2018neural, grathwohl2018ffjord} using the \texttt{dopri5}-option ($\texttt{atol}=10^{-10}, \texttt{rtol}=10^{-5}$) to the DTO implementation given by  \citeauthor{gholami2019anode}  using a fixed grid of 20 steps and 4th-order Runge-Kutta as solver.

\begin{figure}[h]
    \vskip 0.2in
    \begin{center}
    \centerline{\includegraphics[width=\columnwidth]{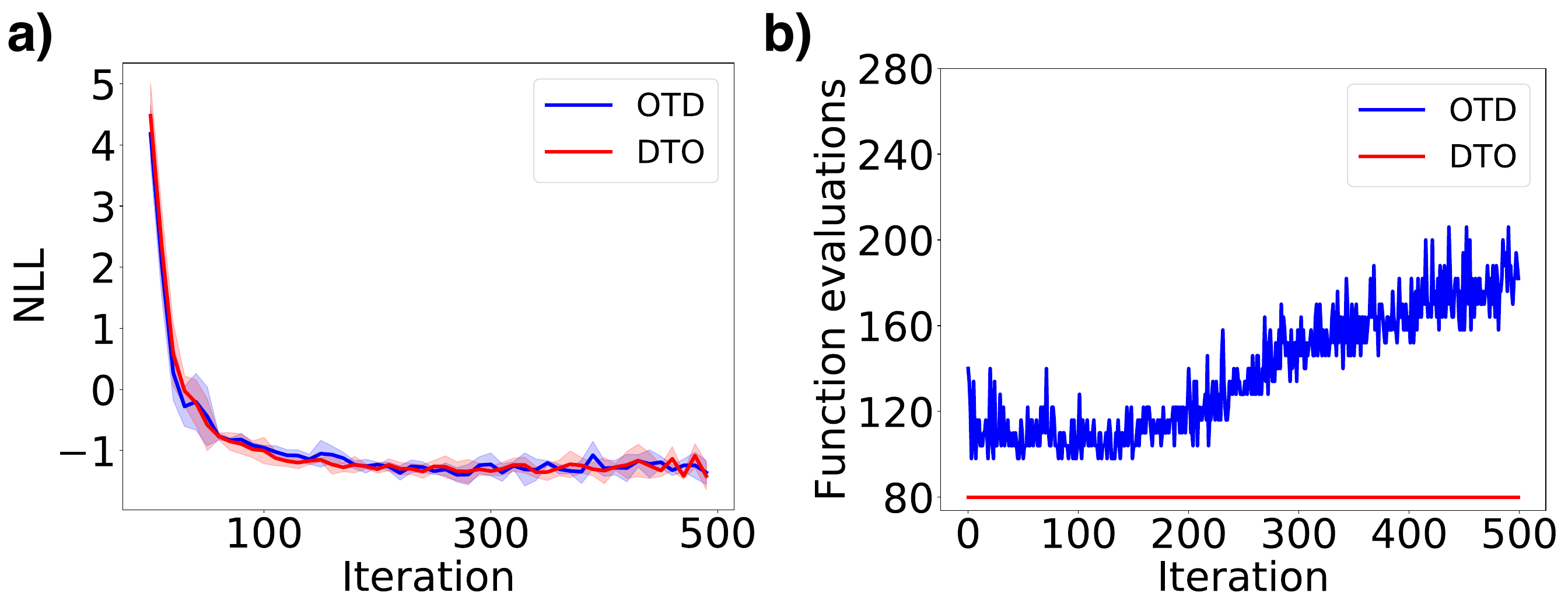}}
  \caption{\textbf{a)} Log-likelihood on test data after training with DTO/OTD for the DW-4 system \textbf{b) } Number of function evaluations increase significantly during training using the OTD approach. The curve for DTO remains flat.}
    \label{fig:otd_dto}
    \end{center}
     \vskip -0.2in
\end{figure}

\subsection{Statistical Efficiency for Density Estimation}\label{sec:density-estimation}

We compare the proposed equivariant flow to a non-equivariant flow where $v(x(t), t)$ is given by a simple fully-connected neural network. As brute-force computation of the divergence quickly becomes prohibitively slow for the LJ-13 system, we rely on Hutchinson-estimation during training and compute the exact divergence only during evaluation.

The training data is generated by taking $10$ / $100$ / $1,000$ / $10,000$ samples from a long MCMC trajectory (throwing away $1,000$ burn-in samples to enforce equilibration). 
After training we evaluate the likelihood of the model on an independent 10,000 trajectory.
We train both flows using Adam with weight decay \cite{kingma2014adam, loshchilov2018fixing} until convergence. For the non-equivariant flow we tested both: data augmentation by applying random rotations and permutations, and no data augmentation.

Our results show that an equivariant flow generalizes well to the unseen trajectory even in the low data regime. When applying data augmentation, the non-equivariant flow significantly performs worse (DW-4) or even fails to fit the data at all and remains close to the prior distribution (LJ-13). Without data augmentation yet using strong regularization we observe strong over-fitting behavior: the DW-4 system can only be fitted if trained on amount of data that is close to the full equilibrium distribution, the LJ-13 system cannot be fitted sufficiently at all (\figref{fig:density_estimation}). It is worth to remark that the equivariant flow only requires $620$ trainable parameters in order to achieve this result compared to the $5256$ (DW-4) / $21671$ (LJ-13) parameters of the black-box model.

\begin{figure}[h]
    \vskip 0.2in
    \begin{center}
    \centerline{\includegraphics[width=\columnwidth]{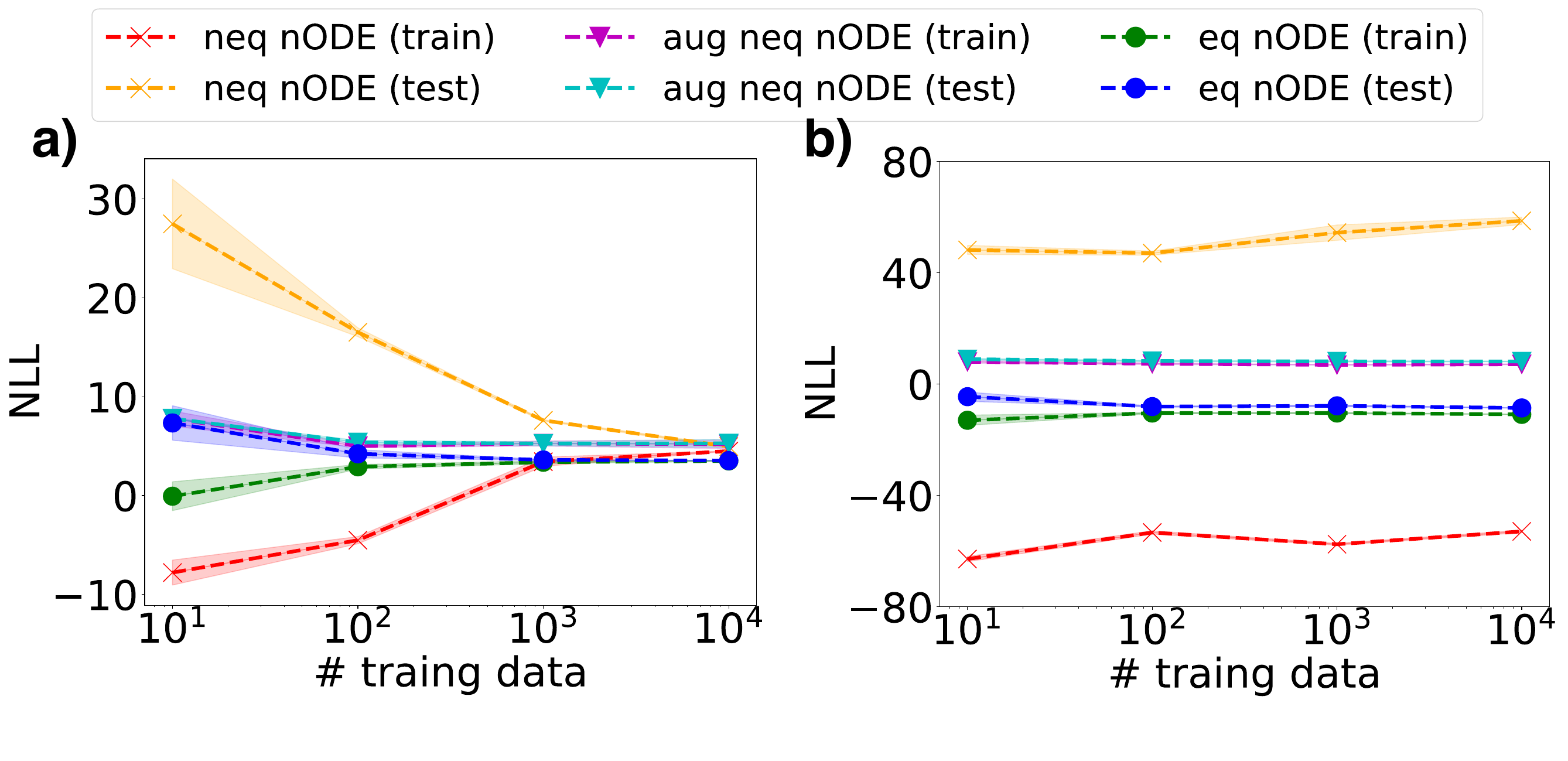}}
    \caption{Log-likelihood on train and test data for both \textbf{a)} the DW-4 and \textbf{b)} the LJ-13 system after training on an increasing number of data points. 
    \textit{eq nODE}: proposed equivariant flow, \textit{neq nODE}: non-equivariant baseline without data augmentation,  \textit{aug neq nODE}: non-equivariant baseline with data augmentation.
    }
    \label{fig:density_estimation}
    \end{center}
     \vskip -0.2in
\end{figure}



\subsection{Equivariance in Boltzmann-Generating Flows}

In a fourth experiment we compare how equivariance affects normalizing flows when used in the context of Boltzman generators (see section \ref{boltzmann-generating-flows} or \cite{noe2019boltzmann}).
For the DW-4 system we compare our equivariant flow to a non-equivariant one when being trained using both: maximizing likelihood on data and minimizing reverse KL-divergence w.r.t. the target density (for details see Suppl. Material). For the non-equivariant flow we tested both: data augmentation and no data augmentation.

The equivariant flow achieves a significant overlap with the target distribution. This allows the target energies to be reweighed to the ground-truth distribution (see \figref{fig:energies_DW_4}~c) and thus to draw asymptotically unbiased samples.
The non-equivariant flow without data augmentation quickly samples low-energy states. However, as indicated by the reweighted distribution and the high train and test likelihood ($10.85$ and $11.40$ respectively), this is due to collapsing to one mode of the distribution (see \figref{fig:energies_DW_4}~a). As a result asymptotically unbiased sampling will not be possible.
The non-equivariant flow after being trained with data augmentation falls short in both: producing accurate low energy states and thus reweighing to the ground-truth (see \figref{fig:energies_DW_4}~b).

\begin{figure}[h]
    \vskip 0.2in
    \begin{center}
    \centerline{\includegraphics[width=\columnwidth]{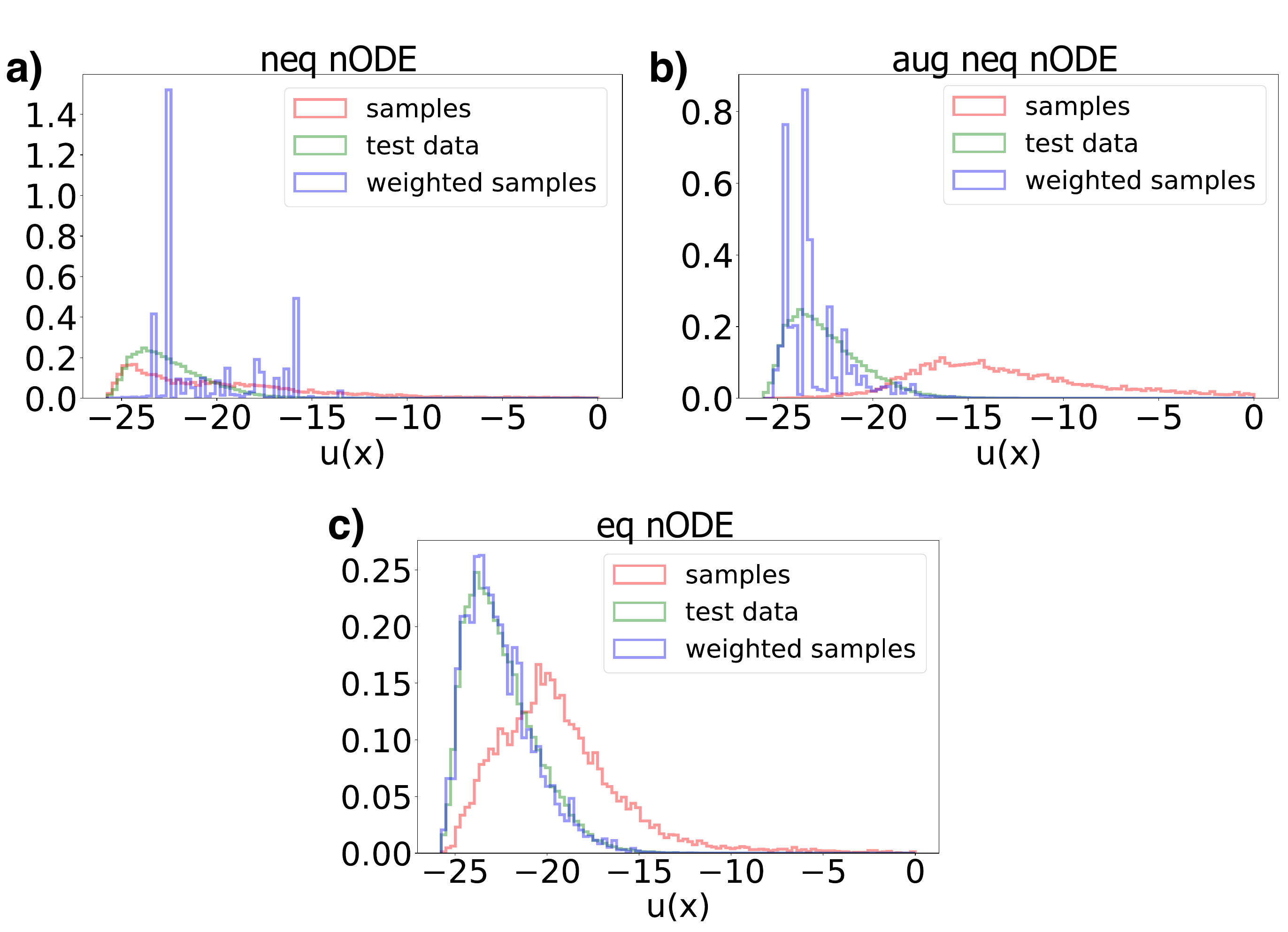}}
    \caption{Energy histograms for samples from the DW-4 system with different models. \textbf{a)} non equivariant nODE, \textbf{b)} non equivariant nODE with data augmentation, and \textbf{c)} proposed equivariant flow.}
    \label{fig:energies_DW_4}
    \end{center}
     \vskip -0.2in
\end{figure}

\subsection{Discovery of Meta-Stable States}\label{sec:Discovery-of-new-meta-stable-states}

In our final experiment, we evaluate to which extend these models help discovering new meta-stable states, which have not been observed in the training data set. Here we characterize metastable states as the set of configurations $x$ that minimize to the same local minimum on the energy surface. Finding new meta-stable states is especially non-trivial for LJ systems with many particles. 

\paragraph{Counting Distinct Meta-Stable States}
Let $\psi$ be the function mapping a state $x$ onto its next meta-stable state $\psi(x)$. 
We implement it by minimizing $x$ w.r.t. $u(x)$ using a non-momentum optimizer until convergence and filtering out saddle-points.
Then we equate two minima $\psi(x) \sim \psi(x')$, whenever they are identical up to rotations and permutations.
To avoid computing the orthogonal Procrustes problem between all minimized structures, we compute the all-distance matrix $M_{d}(\psi(x))$ of each minimum state, sort it in  ascending order to obtain $M_{d, \text{sorted}}(\psi(x))$ and equate two structures $\psi(x) \sim_{\text{approx}} \psi(x')$, whenever ${\|M_{d, \text{sorted}}(x) - M_{d, \text{sorted}}(x')\| < \epsilon},$ where $\epsilon \ll 1$ is a threshold depending on the system. This ensures that $\psi(x) \sim \psi(x') \implies \psi(x) \sim_{\text{approx}} \psi(x')$, however the inverse direction might not hold. Thus, reported numbers on the count of unique minima found remain a lower bound.

\paragraph{DW-4}

For this system, we can fully enumerate those five meta-stable minima between which the system jumps in equilibrium. 
We train both an equivariant flow and a non-equivariant flow on a single minimum state perturbed by a tiny amount of Gaussian noise until convergence. Then we sample $10,000$ structures from both models and compute the set of unique minima.
While the non-equivariant flow model can only reproduce the minimum state it has been trained on, the equivariant flow discovers all minimum states of the system (see \figref{fig:states} a).

\paragraph{LJ-13}




Finding meta-stable minima with low energies is a much more challenging task for the LJ system. Here we compare the proposed equivariant flow to standard sampling by (1) training on a short equilibrium MCMC trajectory consisting of $1,000$ samples, (2) sampling $1,000$ samples from the generator distribution after training, and (3) counting the amount of unique minima states found according to the procedure described above. The amount of unique minima found is compared to sampling an independent equilibrium MCMC trajectory having the same amount of samples as the training set and a long trajectory with $100,000$ samples.

As can be seen from \tableref{tbl:found-minima} the equivariant flow model clearly outperforms naive sampling in finding low-energy meta-stable states compared to the short MCMC trajectory which had access to the same amount of target energy evaluations. Furthermore, in contrast to the latter, it consistently finds the global minimum state, which has not been present in the training trajectory. It performs closely as good as the long trajectory which had access to 100x more evaluations of the target function. \figref{fig:states} b shows structures of low-energy minima generated by the equivariant flow.


\begin{figure}[h]
    \vskip 0.1in
    \begin{center}
    \centerline{\includegraphics[width=.95\columnwidth]{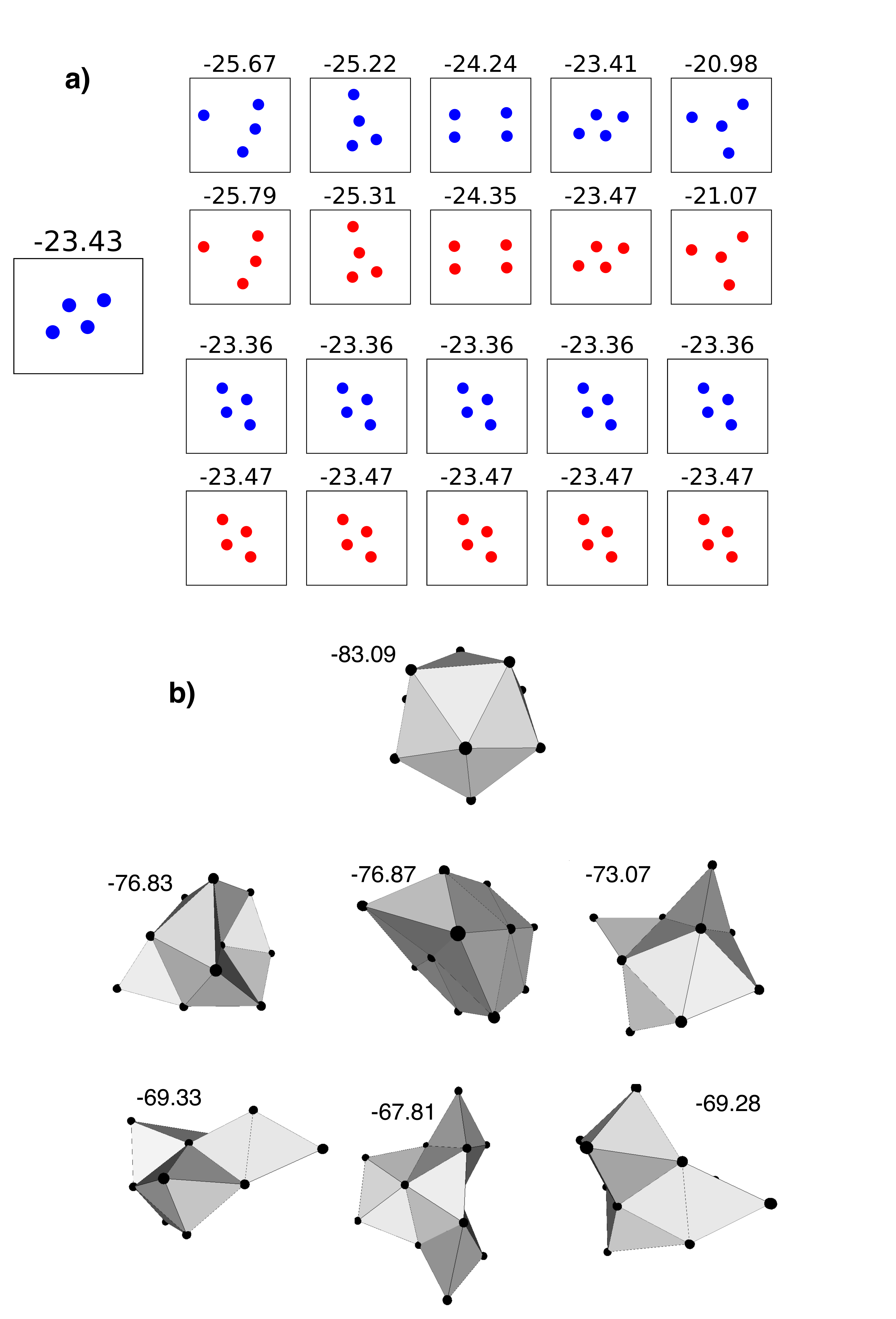}}
    \caption{\textbf{a)} On the left: Minimum state used for training in the DW-4 system. \textit{Upper rows}:  samples from equivariant flow (blue) and corresponding minimum states (red). \textit{Bottom rows}: samples from non-equivariant flow (blue) and corresponding minimum states (red)., \textbf{b)} Exemplary unique minima states from the LJ-13 system generated within the three given energy intervals. The top state marks the global minimum, which consists of a perfect icosahedron with one particle in the center. }
    \label{fig:states}
    \end{center}
     \vskip -0.1in
\end{figure}

\begin{table}[h]
\caption{Count of unique minima states discovered: displayed are means and standard deviation over 10 independent rounds.}
\label{tbl:found-minima}
\vskip 0.15in
\begin{center}
\begin{small}
\begin{sc}
\begin{tabular}{lccccc}
\toprule
& & $u(x)$ \\
Method & $(-70, -60)$ & $(-80, -70)$ & $(-\infty, -80)$\\

\midrule

Training    & $0$& $7$& $0$ \\
Short        & $2.70\pm3.80$& $7.70\pm3.23$ & $0.90\pm0.30$ \\
Long   & $64.60\pm6.11$& $48.60\pm4.13$ & $1.00\pm0.00$ \\
EQ-Flow         & $38.30\pm2.49$& $41.50\pm2.50$ &$1.00\pm0.00$\\

\end{tabular}
\end{sc}
\end{small}
\end{center}
\vskip -0.1in
\end{table}

\section{Discussion}

We presented a construction principle to incorporate symmetries of densities defined over $\mathbb{R}^{n}$ into the structure of normalizing flows. We further demonstrated the superior generalization capabilities of such symmetry-preserving flows compared to non-symmetry-preserving ones on two physics-motivated particle systems, which are difficult to sample with classic methods. Our proposed equivariant gradient field utilizing a simple mixture potential has several structural advantages over black box CNFs, such as an analytically computable divergence, explicit handling of numerical stability and very few parameters.

\section*{Acknowledgements}
We acknowledge funding from European Commission (ERC CoG 772230), Deutsche Forschungsgemeinschaft (SFB1114/A04, DAEDALUS/P04), MATH+ The Berlin Mathematics Research Center (AA1-6, EF1-2).

We further thank Moritz Hoffmann (FU Berlin), Pim de Haan (UvA Amsterdam / Qualcomm) and Onur Caylak (TU Eindhoven) for helpful remarks and discussions.

\clearpage

\bibliography{example_paper}
\bibliographystyle{icml2020}

\clearpage
\newpage
\onecolumn
\appendix
{\Large \textbf{Supplementary Material}}

\section{Proofs and derivations}
\subsection{Proof of thm. \ref{thm:sufficient-criterion-classic-flow}}

Let $V=\mathbb{R}^{n}$ and $\rho \colon V \rightarrow \mathbb{R}_{\geq 0}$ be a probability density on $V$. Let $G$ be a group acting on $V$ and let $R \colon G \rightarrow GL(n),g \rightarrow R_g$ be a representation of $G$ over $V$. As $V$ is finite-dimensional every $R_g$ is represented by a matrix and thus $\det R_g$ is well-defined. Furthermore, for a function $f \in C^{1}(\mathbb{R}^n, \mathbb{R}^m)$ let $J_{f}(x) \in \mathbb{R}^{n \times m}$ denote its Jacobian evaluated at $x$ and define the push-forward density of $\rho$ along a diffeomorphism $f \in C^{1}(V, V)$ by $\rho_{f}(x) := \rho(f^{-1}(x)) \left| \det J_{f^{-1}}(x) \right|$.

\begin{lemma}
    Let $A \in GL(n)$, if $\rho(A x) = \rho(x)$ for all $x \in V$, then $\det A \in \left\{ -1, 1\right\}$
\end{lemma}
\begin{proof}
    Set $a\colon V \to V, x \mapsto A x$. By substituting $y = a^{-1} x$ we get
    \begin{align*}
        1 &= \int_{V} \rho(x) dx \\
        &= \int_{a^{-1} (V)} \rho(a(y)) \left| \det A \right| dy\\
        &= \int_{V} \rho(y) \left| \det A \right| dy\\
        &= \left| \det A \right| \underbrace{\int_{V} \rho(z) dy}_{=1} \\
        &= \left| \det A \right|
    \end{align*}
\end{proof}

Let $G>H$ and $h \in H$. From Lemma 1 we get $\det R_h \in \left\{-1, 1\right\}$ for each $h \in H$.  Define the transformation $T_{h} \colon V \rightarrow V, ~ x \mapsto R_h x$.
If $f \in C^{1}(V, V)$ is $H$-equivariant, it means $f \circ T_{h} = T_{h} \circ f$ for each $h \in H$. If $\rho$ is an $G$-invariant density it means $\rho \circ T_g = \rho$.  Together with the lemma we obtain

\begin{align*}
    \rho_{f}(R_h x)
    &= \rho_{f}(T_{h}(x)) \\
    &=  \rho_{f}(T_{h}(x)) \underbrace{\left| \det J_{T_{h}} (x) \right|}_{=\left| \det R_h \right| = 1}\\
    &= \rho_{T_{h^{-1}} \circ f}(x) \\
    &= \rho_{f \circ T_{h^{-1}}} (x)\\
    &= \rho((T_h \circ f^{-1})(x)) \left| \det J_{T_h \circ f^{-1}}(x) \right| \\
    &= (\rho \circ T_h \circ f^{-1})(x) \left| \det J_{T_h}(f^{-1}(x)) J_{f^{-1}}(x) \right| \\
    &= (\rho\circ f^{-1})(x) \left| \det J_{T_h}(f^{-1}(x)) \right| \left| \det J_{f^{-1}}(x) \right| \\
    &= \rho( f^{-1}(x)) \left| \det R_h \right| \left| \det J_{f^{-1}}(x) \right| \\
    &= \rho( f^{-1}(x)) \left| \det J_{f^{-1}}(x) \right| \\
    &= \rho_{f}(x) &\qedhere
\end{align*}

\subsection{Proof of thm. \ref{thm:sufficient-criterion-continuous-flow}}

\begin{proof}
Let $h \in H$ and $R_h$ be its representation. Let $v$ be an $H$-equivariant vector field. Then
\begin{align*}
    F_{v,T}(R_h z) 
    &= R_h x_{v, z}(0) + \int_{0}^{T}  dt ~  v(R_h x_{v,z}(t), t) \\
    &= R_h x_{v, z}(0) +  \int_{0}^{T}  dt ~  R_h v(x_{v,z}(t), t)\\
    &= R_h \left( x_{v, z}(0) + \int_{0}^{T}  dt ~  v(x_{v,z}(t), t) \right).
\end{align*}
This implies that the bijection $F_{v,T}$ for each $T\in [0, \infty)$  given by solving 
\begin{align*}
    x_{v, z}(0) &= z \\
    \frac{d}{d t} x(t)&= v(x_{v,z}(t), t)
\end{align*}
is $H$-equivariant. 
\end{proof}

\subsection{Invariant prior density}

Subtracting the CoM of a system $x \in \mathbb{R}^{N \cdot D}$ and obtaining a CoM-free $\tilde{x}$, can be considered a linear transformation
$$ \tilde{x} = A x $$
with 
$$ 
A = \textbf{I}_{D} \otimes \left(\textbf{I}_{N} - \tfrac{1}{N} \textbf{1}_{N} \textbf{1}_{N}^{T}\right) 
$$
where $\textbf{I}_{k}$ is the $k \times k$ identity matrix and $\textbf{1}_{k}$ the $k$-dimensional vector containing all ones.

$A$ is a symmetric projection operator, i.e. $A^2 = A$ and $A^T = A$. Furthermore $\text{rank}\left[A\right] = (N-1) D$. Finally, we have $Ay = y$ for each $y \in U$.

If we equip $\mathbb{R}^{n}$ with an isotropic density $\rho = \mathcal{N}(\mathbf{0}, \mathbf{I}_{n})$, this implies the subspace density $\tilde \rho = \mathcal{N}(\mathbf{0}, A \mathbf{I}_{n} A^{T}) = \mathcal{N}(\mathbf{0}, A A^{T})$.
Thus, sampling from $\rho$ and projecting by $A$ achieves sampling from $\tilde \rho$ trivially. On the other hand, if we have $y \in U$, then $\| y \|^{2}_{2} = \| A y \|^{2}_{2}$ and thus $\rho(y) = \tilde\rho(y)$.

If $f$ is an equivariant flow w.r.t. symmetries (S1-3) we see that any CoM-free system is mapped onto another CoM-free system and thus defines a well-defined flow on the subspace spanned by $A$. 

\subsection{Derivations for the RBF gradient field}

We first show that $v$ as defined in \eqref{eq:gradient-field} is indeed a gradient field. Define
\begin{align}
    \alpha(x, a, b) = \sqrt{\frac{\pi a b^2}{2}} \text{erf}\left(\frac{x - b}{\sqrt{2 a}}\right) - a \exp\left( - \frac{(x-b)^2}{2 a} \right),
\end{align}
where $\text{erf}$ denotes the Gaussian error function. Then we have
\begin{align}
    \frac{\partial \alpha(x, a, b)}{\partial x} = \exp\left(- \frac{(x- b)^2}{2 a}\right) \cdot x.
\end{align}
Now by setting 
\begin{align}
    \kappa(d) = \frac{1}{2} \left( \begin{array}{ccc}
    \alpha(d, \mu_{1}, \sigma_{1}) & \ldots & \alpha(d, \mu_{M}, \sigma_{M})
    \end{array}\right) 
\end{align}
and
\begin{align}
    \tilde \Phi(d_{ij}, t) = R(t) W \kappa(d_{ij})
\end{align}
we obtain
\begin{align}
    \frac{\partial \tilde \Phi(d_{ij}, t)}{ \partial x_{i}} &= R(t) W \frac{\partial \kappa(d_{ij})}{\partial d_{ij}} \frac{\partial d_{ij}}{\partial x_{i}} \\
    &= R(t) W \frac{\partial \kappa(d_{ij})}{\partial d_{ij}} \frac{r_{ij}}{d_{ij}} \\
    &= \frac{1}{2} R(t) W K(d_{ij}) r_{ij}
\end{align}
where 
\begin{align}
    K(d) = \left( \begin{array}{ccc}
    \exp\left(- \frac{(d- \mu_{1})^2}{2 \sigma_{1}}\right) & \ldots & \exp\left(- \frac{(d- \mu_{M})^2}{2 \sigma_{M}}\right)
    \end{array}\right).
\end{align}
Similarly, we have
\begin{align}
\frac{\partial \tilde \Phi(d_{ji}, t)}{ \partial x_{i}} &= - \frac{1}{2}  R(t) W K(d_{ji}) r_{ji} \\
&= \frac{1}{2}  R(t) W K(d_{ij}) r_{ij}.
\end{align}
Thus, we obtain
\begin{align}
    \left( \nabla_{x}\Phi(x, t) \right)_{i} &= \frac{\partial \Phi(x, t)}{\partial x_{i}} \\
    &= \sum_{lj} \frac{\partial \tilde \Phi(d_{lj}, t)}{\partial x_{i}} \\
    &= \sum_{j} R(t) W K(d_{ij}) r_{ij} \\
    &=\sum_{j}\phi(d_{ij}) r_{ij}
\end{align}
Finally, by using $v(x, t) = \tfrac{\partial x(t)}{\partial t} = \nabla_{x} \Phi(x, t)$ we can compute the divergence as
\begin{align}
    \divergence{v(x, t)} 
    &= \trace{ \frac{ \partial v(x, t) }{\partial x}} \\
    &= \sum_{i} \trace{\frac{\partial v_{i}(x, t)}{\partial x_{i}}}\\
    &= \sum_{ij} \trace{ \frac{\partial \phi(d_{ij})}{\partial d_{ij}} r_{ij} \frac{\partial d_{ij}}{\partial x_{i}}^T + \phi(d_{ij}) \frac{\partial r_{ij}}{\partial x_{i}} } \\
    &= \sum_{ij} \frac{\partial \phi(d_{ij})}{\partial d_{ij}} \trace{ r_{ij} \frac{r_{ij}}{d_{ij}} ^T }  + \phi(d_{ij}) \trace{ I_{D \times D} } \\
    &= \sum_{ij} \frac{\partial \phi(d_{ij})}{\partial d_{ij}} \frac{r_{ij}^T r_{ij}}{d_{ij}}    + \phi(d_{ij}) D \\
    &= \sum_{ij} \frac{\partial \phi(d_{ij})}{\partial d_{ij}} d_{ij}  + \phi(d_{ij}) D.
\end{align}


\section{Additional experiments}

\subsection{Comparison with other equivariant gradient flows}
Next to the equivariant flow as proposed in the main text (furthermore referred to as \textit{Kernel Flow}) we experimented with two neural-network-based equivariant flow architectures relying on the CNF framework. Both performed inferior in terms of accuracy and execution time. 

One ad-hoc way to model a rotation and permutation invariant potential $\Phi$ matching the given target systems could be given by only feeding pairwise particle distances to $\Phi$. Similarly to the architecture presented in the main text, we can model $\tilde \Phi(d_{ij}(t),t)$, where $\Phi(x(t)) = \sum_{ij} \tilde \Phi(d_{ij}(t),t)$ and then take its gradient field for the CNF. To this end we embed the distances with Gaussian RBF-kernels and use a simple multi-layer perceptron to obtain a scalar output for each embedded distance. 
We refer to this flow as \textit{simple gradient flow}.

Another invariant, possibly more complex, potential $\Phi$ can be obtained by using a molecular message passing architecture similar to \textit{SchNet} \cite{schutt2017schnet}. We refer to this flow as \textit{gradient flow with SchNet}. For the implementation details see \ref{sec:eq-gradient-flows}.

As before, the Hutchinson estimator is used during training, while sampling and reweighing is done brute-force. 

In this first additional experiment we compare training / sampling wall-clock-times for the equivariant gradient flows and the kernel flow. To this end we measure the time per iteration for training / sampling with a batch / sample size of $64$.
Our results show that kernel flows are considerably faster during training and at least one order of magnitude faster during sampling  (\tableref{tbl:iteration-time}).
This speedup during training is mainly rooted in  AD required to calculate $\nabla_{x(t)} \Phi(x(t))$. The overhead during sampling is a consequence of brute-force computing $\Delta_{x(t)} \Phi(x(t))$.

\begin{table}[ht]
\caption{Training and sampling wall-clock-times in seconds: displayed are means and standard deviation over 10 epochs with 10 iterations each and a batch size of $64$.}
\label{tbl:iteration-time}
\vskip 0.15in
\begin{center}
\begin{small}
\begin{sc}
\begin{tabular}{lcccc}
\toprule
& \multicolumn{2}{c}{Training}& \multicolumn{2}{c}{Sampling}\\
Model & DW-4 & LJ-13  & DW-4 & LJ-13 \\

\midrule

kernel flow    & $\mathbf{0.879 \pm 0.012 }$& $\mathbf{0.498 \pm 0.010}$& $\mathbf{0.222 \pm 0.003}$& $\mathbf{0.178 \pm 0.002}$\\
simple gradient flow & $2.58 \pm 0.091$& $ 1.760 \pm 0.056$  & $4.18 \pm 0.112$ & $6.670 \pm 0.269$\\
gradient flow with SchNet  & $ 6.140 \pm 0.583$& $ 4.770 \pm 0.500$&  $7.690 \pm 0.018$& $25.400 \pm 3.690$\\
\end{tabular}
\end{sc}
\end{small}
\end{center}
\vskip -0.1in
\end{table}

\subsection{Density estimation with coupling flows and other equivariant gradient flows}
In a second additional experiment we compare density estimation of the equivariant flow as proposed in the main text to other equivariant gradient flows and flows based on non-equivariant coupling layers. For the coupling layers we use real NVP (rNVP) transformations as introduced by \cite{dinh2016density} and used for Boltzmann Generators in \cite{noe2019boltzmann}. The experiment follow section \ref{sec:density-estimation}. 
As coupling layers do not conserve the CoM we add a Gaussian distributed CoM to the training data. We further use a Gaussian prior with the same CoM in the latent space. Equivariant / CoM-preserving flows will not change the CoM. Thus running an equivariant flow on the CoM-free system and then adding then energy of the CoM perturbation yields the negative log-likelihood corresponding to running the rNVP flow on the perturbed system.

The other equivariant gradient flows achieve similar albeit slightly larger log-likelihoods on the training and test set (\figref{fig:density_estimation_eq}).  
The non equivariant flow based on rNVP achieves a log-likelihood on the test similar to the non equivariant nODEs  (\figref{fig:density_estimation_rnvp}) when data augmentation is applied. Without data augmentation the log-likelihood of the rNVP between train and test set differs in the range of $\sim 200$.

\begin{figure}[ht]
    \vskip 0.2in
    \begin{center}
    \centerline{\includegraphics[width=0.8\columnwidth]{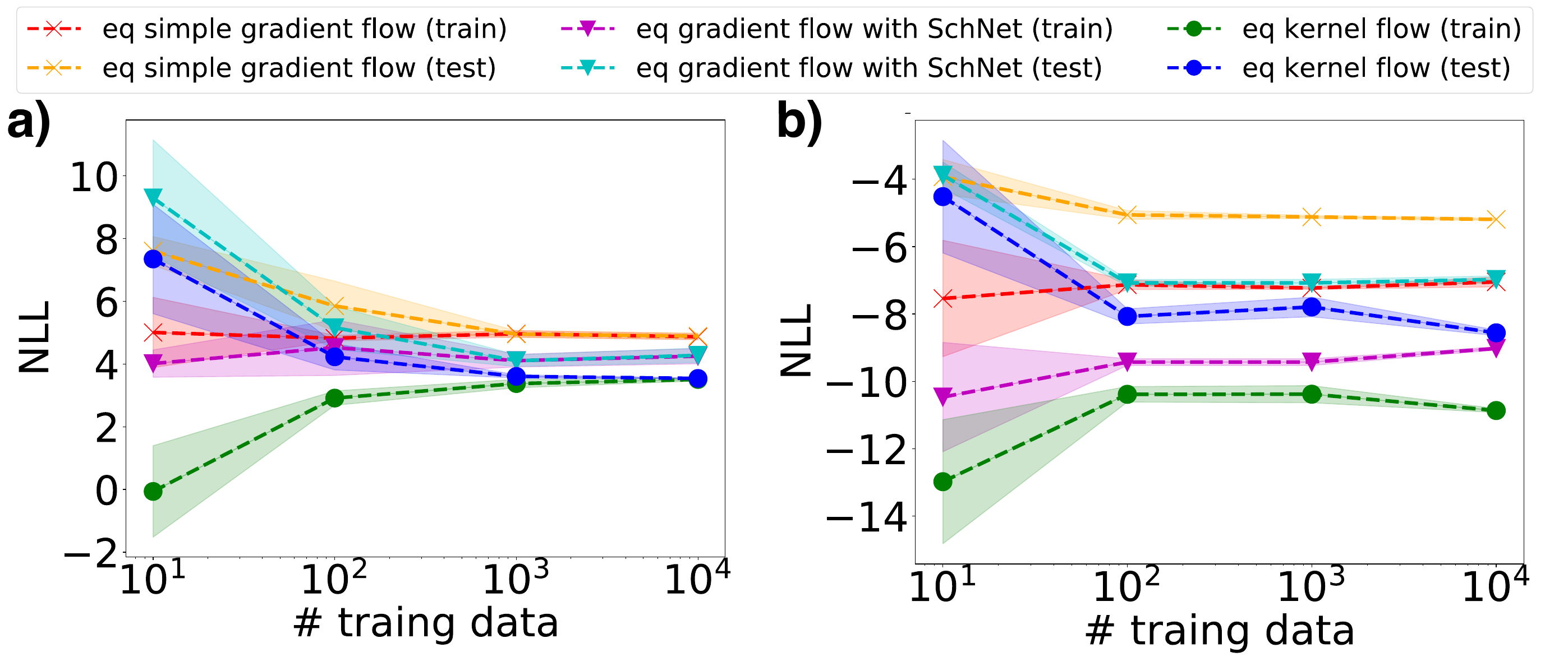}}
    \caption{Log-likelihood on train and test data for both \textbf{a)} the DW-4 and \textbf{b)} the LJ-13 system after training on an increasing number of data points. All equivariant flows generalizes quickly to unseen trajectories, but the kernel flow achieves the lowest log-likelihoods on the train and test data.}
    \label{fig:density_estimation_eq}
    \end{center}
     \vskip -0.2in
\end{figure}

\begin{figure}[ht]
    \vskip 0.2in
    \begin{center}
    \centerline{\includegraphics[width=0.7\columnwidth]{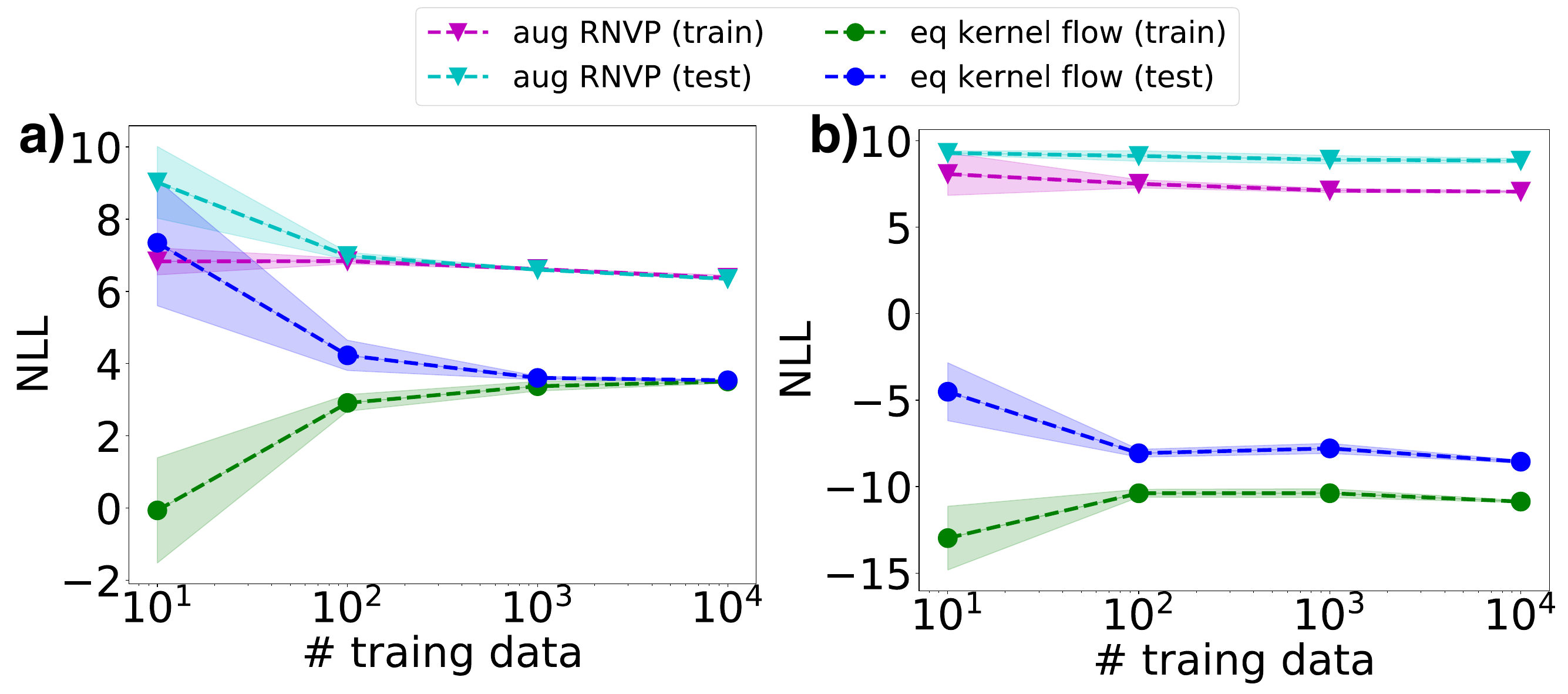}}
    \caption{Log-likelihood on train and test data for both \textbf{a)} the DW-4 and \textbf{b)} the LJ-13 system after training on an increasing number of data points. The non-equivariant coupling flow with data augmentation (\textit{aug RNVP})  performs significantly worse compared to the equivariant flow on both test systems. For the LJ-13 system it is even  unable to fit the augmented data at all and remains close to the prior.}
    \label{fig:density_estimation_rnvp}
    \end{center}
     \vskip -0.2in
\end{figure}

\subsection{Energy based training for the DW-4 system}
This experiment follows a similar training procedure as done in \cite{noe2019boltzmann}. The models are pretrained with \textit{ML-training} on $1000$ samples from a long MCMC trajectory. Then the we use a combination of \textit{ML-} and \textit{KL-training} (see sec \ref{sec:related-work} bottom or \cite{noe2019boltzmann}), where we increase the proportion of the KL-loss from $\lambda= 0$ to $\lambda=0.5$ over the course of training. Aftwerwards we sample $10,000$ samples and compare their energy and reweighted energy to the expected energy of the samples from the long MCMC trajectory (\figref{fig:energies}). The non-equivariant models without data augmentation are capable of producing samples with low energies, but the reweighing fails in these cases (\figref{fig:energies}~d, f). This is due to mode collapsing onto a small part of the target space, i.e. most of the produced samples are only from a single rotation/permutation of a certain configuration. Hence, a random sample lying in another region will yield a very large reweighing weight. This behaviour was also observed in \cite{noe2019boltzmann} and gets worse with larger proportions of the KL-loss. The non-equivariant models with data augmentation produce samples with high energies making reweighing difficult as well (\figref{fig:energies}~e, g). These models are prone to mode collapse as well if the proportion of the KL-loss becomes larger than $\lambda=0.5$. We thus chose $\lambda=0.5$ throughout the reported experiments.  
In contrast, the equivariant models (kernel flow, simple gradient flow and the gradient flow with SchNet) produce low energies and allow for correct reweighing (\figref{fig:energies}~a, b, c).

\begin{figure}[ht]
    \vskip 0.2in
    \begin{center}
    \centerline{\includegraphics[width=\columnwidth]{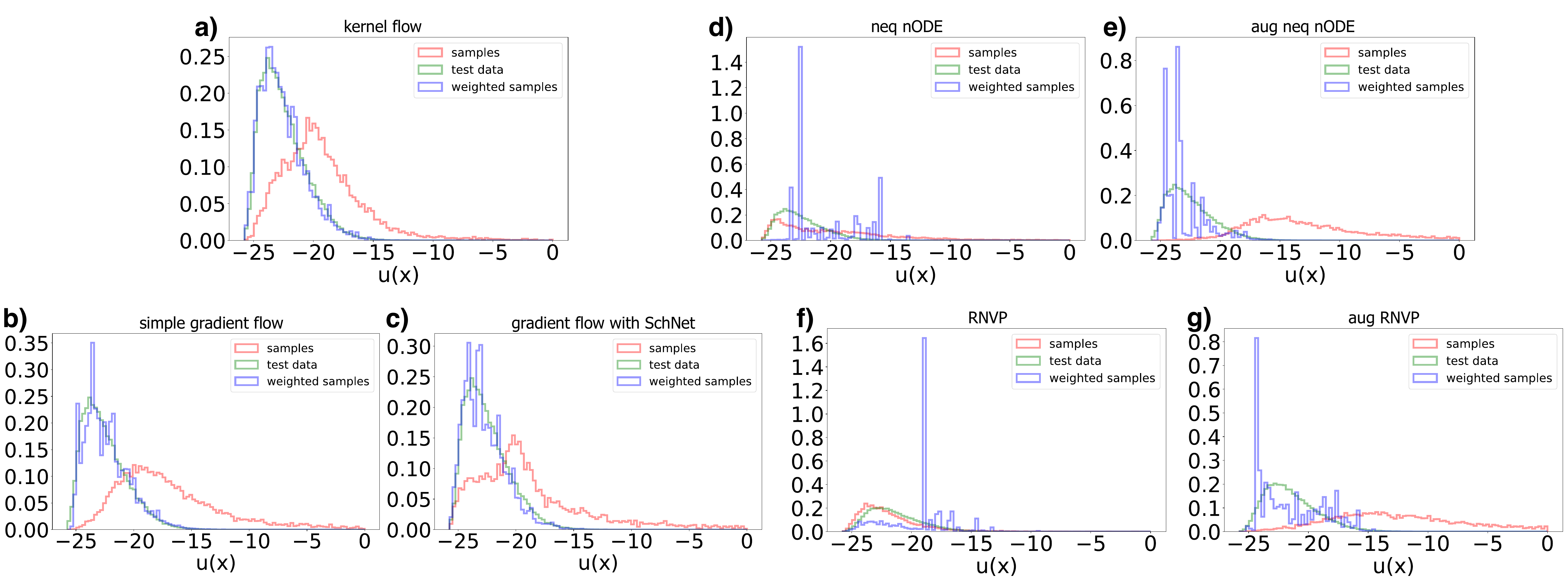}}
    \caption{Energy histograms for samples from the DW-4 system with different models.}
    \label{fig:energies}
    \end{center}
     \vskip -0.2in
\end{figure}

\subsection{Visualisation of the weights of the kernel flow}

Due to the simple structure of the kernel flow we can visualize the learned dynamics of the equivariant flow, by plotting $W$ after training (\figref{fig:learned_dynamics}).

\begin{figure}[ht]
    \vskip 0.2in
    \begin{center}
    \centerline{\includegraphics[width=\columnwidth]{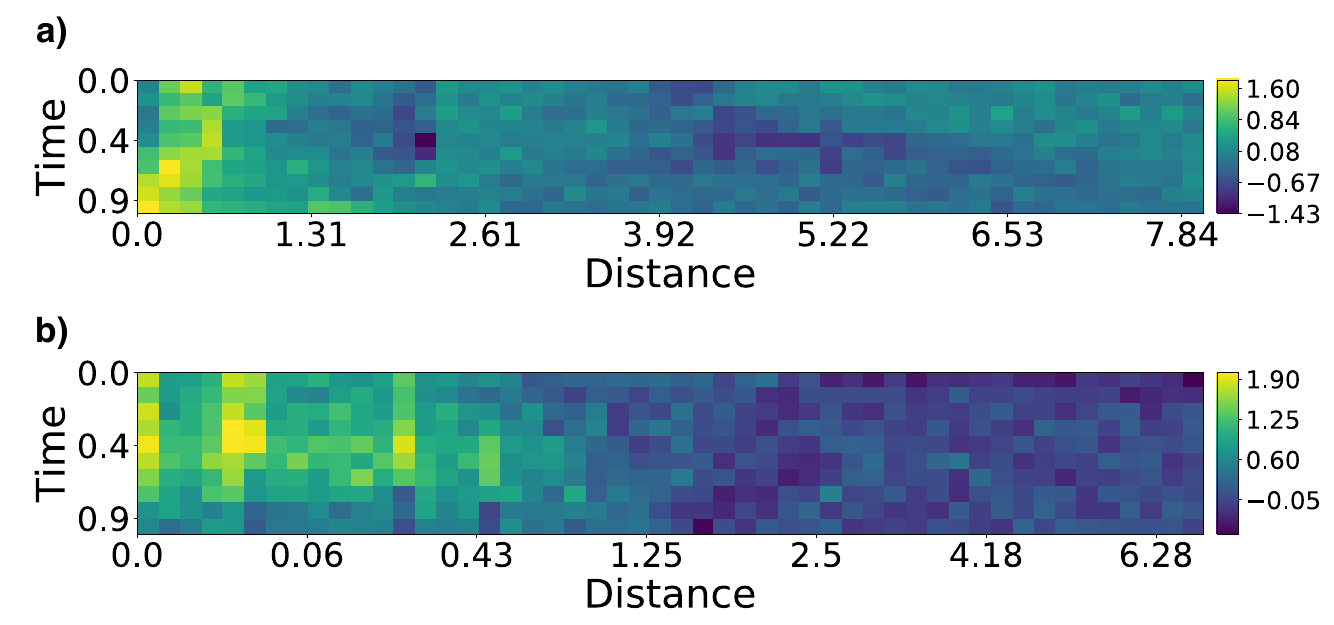}}
    \caption{Kernel weights $W$ visualized after training for the  \textbf{a)} DW-4 system and the \textbf{b)} LJ-13 system. }
    \label{fig:learned_dynamics}
    \end{center}
     \vskip -0.2in
\end{figure}

\section{Technical details}
In this section we show the hyperparameters and optimization details used for the experiments presented in this work.
For all experiments we used ML-training to train the models on the given training data ($\lambda=0$) if not stated otherwise. For the discovery of new meta-stable states (\ref{sec:Discovery-of-new-meta-stable-states}) for the \textit{DW-4} system, where we used a combination of  \textit{ML-training} and \textit{KL-training} with $\lambda = 0.5$ after pretraining both models with \textit{ML-training}.

\subsection{Equivariant kernel flow}
For the \textit{DW-4} system we fixed $50$, $10$ kernel means $\mu_{K,l}$, $\mu_{R,l}$ equispaced in $[0, 8]$, $[0,1]$ for distances and times respectively. The bandwidths $\gamma_{K, l}, \gamma_{R, l}$ of the kernels have been initialized with $0.5$, $0.3$ and were optimized during the training process. The total model ended up having $620$ trainable parameters.

For the \textit{LJ-13} system we fixed $50$ kernel means $\mu_{K,l}$ in $[0, 16]$ concentrated around $r_{m}=1$ with increasing distance to each other towards the interval bounds. Similarly bandwidths $\gamma_{K, l}$ are initialized narrowly close to $r_{m}$ increasing towards the interval bounds. We placed the $10$ kernels $\mu_{R, l}$  for the time-dependent component equispaced in $[0, 1]$. The bandwidths $\gamma{R, l}$ where initialized with narrower bandwidths around $t=0.5$ and smearing out the closer they reach the interval boundaries. Again bandwidths were optimized during the training process. This resulted in a total of $620$ trainable parameters. 

As regularization is important to efficiently train our architecture using fixed step-size solvers our models were optimized using \texttt{AdamW}, a modified implementation with fixed weight-decay \citep{loshchilov2017fixing} using a learning rate of $0.005$, weight decay of $0.01$ and a batch size of $64$ samples until convergence.

\subsection{Equivariant gradient flows}\label{sec:eq-gradient-flows}
We used the same distance embedding for both equivariant gradient flows.

For the \textit{DW-4} system we fixed $50$ rbf-kernel means equispaced in $[0, 8]$. The bandwidths have been initialized with $0.5$ and were not optimized during the training process.

For the \textit{LJ-13} system we fixed $50$ rbf-kernel means in $[0, 16]$ concentrated around $r_{m}=1$ with increasing distance to each other towards the interval bounds. Similarly bandwidths $\gamma_{K, l}$ are initialized narrowly close to $r_{m}$ increasing towards the interval bounds and were not optimized during the training process..

For the equivariant gradient flows with the simple potential the transformation of each embedded distance was modeled with a dense neural network with layer sizes $[50, 64, 32, 1]$ and \texttt{tanh} activation functions. This resulted in a total of $5377$ trainable parameters.

For the equivariant gradient flows with the SchNet \cite{schutt2017schnet} inspired potential we used $16$ features and 3 interaction blocks. For the feature encoding we used a simple network with layer sizes $[1, 16]$. For the continuous convolutions in the three interaction blocks we used dense neural networks with layer sizes $[50, 32, 32, 16]$ and \texttt{tanh} activation functions. Finally, a dense neural networks with layer sizes $[16, 8, 4, 1]$ was used to compute the invariant energy after the interaction blocks. This resulted in a total of $9857$ trainable parameters.

\subsection{Non equivariant nODE flow}
For the \textit{DW-4} system we used a dense neural network with layer sizes  $[64, 64]$ and \texttt{tanh} activation functions.
This resulted in a total of  $5256$  trainable parameters.

For the \textit{LJ-13} system we used a dense neural network with layer sizes $[64, 128, 64]$ and \texttt{tanh} activation functions.
This resulted in a total of $21671$ trainable parameters. For the optimization we used \texttt{AdamW} with a learning rate of $0.005$. 
We optimized the model with a batch size of $64$ samples until convergence. 

\subsection{MCMC trajectories}
For each system, a training and a test trajectory were obtained with Metropolis Monte-Carlo, where we optimized the width of the Gaussian proposal density by maximizing $\alpha \cdot s$, with $\alpha$ being the acceptance rate computed from short trajectories and $s$ the Gaussian standard deviation (step size). The optimal step sizes are $s=0.5$ for the $DW-4$ system and $s=0.025$ for the $LJ-13$ system. 
To ensure that all samples steam from the equilibrium distribution we discard a large number of initial samples. For the \textit{DW-4} system  the initial $1000$ samples are discarded, while we discard $20000$ for the \textit{LJ-13} system. 

\subsection{Non equivariant coupling flows}
For the \textit{DW-4} system we used $8$ coupling blocks. For the translation transformation we used a dense neural network with layer sizes $[4, 64, 64, 4]$ and \texttt{ReLU} activation functions. For the scaling transformation we used a dense neural network with layer sizes $[4, 64, 64, 4]$ and \texttt{tanh} activation functions.
This resulted in a total of $21576$ trainable parameters.

For the \textit{LJ-13} system we used $16$ coupling blocks. For the translation transformation we used a dense neural network with layer sizes $[19/20, 64, 64, 20/19]$ and \texttt{ReLU} activation functions. For the scaling transformation we used a dense neural network with layer sizes $[19/20, 64, 64, 20/19]$ and \texttt{tanh} activation functions. The number of input and output neurons for each network is either $19$ or $20$ due to the uneven number of total dimensions.  
This resulted in a total of $215680$ trainable parameters.


\subsection{Benchmark systems}
Throughout all experiments we chose the same parameters for our two benchmark systems. 

For the \textit{DW-2} / \textit{DW-4} system we chose $a=0, b=-4, c=0.9, d_0=4$ and a dimensionless temperature factor of $\tau=1$.

For the \textit{LJ-13} system we chose $r_m=1, \epsilon=1$ and a dimensionless temperature factor of $\tau=1$.

\subsection{Error bars}
Error bars in all plots are given by one standard deviation.

In \figref{fig:trace} a) we show errors for $1000$ estimations per particle count.
In \figref{fig:trace} b) errors are displayed for $100$ reweighed bootstrapped sub-samples.
In \figref{fig:trace} c) time was measured for $100$ estimations per particle count per method. 

In \figref{fig:otd_dto} a) we show $3$ runs per method.

In \figref{fig:density_estimation}, \figref{fig:density_estimation_eq}, and \figref{fig:density_estimation_rnvp} we show $5$ runs per model/system/training set size.

\subsection{Computing infrastructure}
All experiments were conducted on a \textit{GeForce GTX 1080 Ti} with 12 GB RAM.

\newpage




\end{document}